\documentclass[10pt,journal,letterpaper,twoside]{IEEEtran}

\usepackage[T1]{fontenc}
\usepackage{newtxtext}
\usepackage{amsmath,amssymb,mathtools}
\usepackage{amsthm}
\usepackage{newtxmath}
\usepackage{booktabs,multirow,array}
\usepackage[table]{xcolor}
\usepackage{colortbl}
\usepackage{algorithm}
\usepackage[noend]{algpseudocode}
\usepackage{enumitem}
\usepackage{xspace}
\usepackage{microtype}
\usepackage{graphicx}
\usepackage{pifont}
\usepackage{listings}
\usepackage{url}
\usepackage{cite}
\usepackage{placeins}
\usepackage[most]{tcolorbox}
\usepackage[hidelinks]{hyperref}

\graphicspath{{./}{fig/}{logos/}}
\makeatletter
\def\input@path{{./}}
\makeatother

\definecolor{Ink}{HTML}{20262E}
\definecolor{Navy}{HTML}{274C77}
\definecolor{Teal}{HTML}{2A8178}
\definecolor{Orange}{HTML}{C66B2B}
\definecolor{Brick}{HTML}{A64B3C}
\definecolor{Line}{HTML}{C8D2D9}
\definecolor{PaleBlue}{HTML}{F0F5F8}
\definecolor{PaleTeal}{HTML}{F1F8F5}
\definecolor{PaleGray}{HTML}{F5F6F7}
\definecolor{PaleOrange}{HTML}{FFF5ED}
\definecolor{PreprintBlue}{HTML}{28658A}
\definecolor{PreprintBlueDark}{HTML}{1F4F6C}
\definecolor{PreprintGreen}{HTML}{2C805B}
\definecolor{PreprintGreenDark}{HTML}{226548}
\definecolor{PreprintOrange}{HTML}{C5672D}
\definecolor{PreprintRule}{HTML}{626A70}
\definecolor{SkillZipBlue}{RGB}{35,79,119}
\definecolor{SkillZipGreen}{RGB}{45,105,78}
\definecolor{SkillZipPaleBlue}{RGB}{232,240,247}
\definecolor{SkillZipPaleGray}{RGB}{245,246,247}

\newtcolorbox{abstractbox}{
  enhanced,
  breakable,
  colback=PaleBlue,
  colframe=PreprintBlue,
  boxrule=0.55pt,
  arc=2pt,
  left=5pt,
  right=5pt,
  top=4pt,
  bottom=4pt,
  before skip=0pt,
  after skip=5pt,
  fontupper=\small
}

\newtcolorbox{projectbox}{
  enhanced,
  breakable,
  colback=PaleTeal,
  colframe=PreprintGreen,
  boxrule=0.55pt,
  arc=2pt,
  left=5pt,
  right=5pt,
  top=3pt,
  bottom=3pt,
  before skip=0pt,
  after skip=7pt,
  fontupper=\small
}

\newtcolorbox{glancebox}{
  enhanced,
  breakable,
  title={SkillZip at a glance},
  colback=PaleOrange,
  colframe=PreprintOrange,
  colbacktitle=PreprintOrange,
  coltitle=white,
  fonttitle=\bfseries\sffamily\footnotesize,
  fontupper=\small,
  boxrule=0.55pt,
  arc=2pt,
  left=5pt,
  right=5pt,
  top=3pt,
  bottom=3pt,
  toptitle=2pt,
  bottomtitle=2pt,
  before skip=5pt,
  after skip=7pt
}

\newtcolorbox{takeawaybox}{
  enhanced,
  title={Takeaway},
  colback=PaleTeal,
  colframe=PreprintGreen,
  colbacktitle=PreprintGreen,
  coltitle=white,
  fonttitle=\bfseries\sffamily\footnotesize,
  fontupper=\small,
  boxrule=0.55pt,
  arc=2pt,
  left=5pt,
  right=5pt,
  top=3pt,
  bottom=3pt,
  toptitle=2pt,
  bottomtitle=2pt,
  before skip=5pt,
  after skip=6pt
}

\newtcolorbox{roadmapbox}{
  enhanced,
  breakable,
  title={Appendix roadmap},
  colback=PaleBlue,
  colframe=PreprintBlue,
  colbacktitle=PreprintBlue,
  coltitle=white,
  fonttitle=\bfseries\sffamily\small,
  fontupper=\small,
  boxrule=0.55pt,
  arc=2pt,
  left=6pt,
  right=6pt,
  top=4pt,
  bottom=4pt,
  toptitle=2pt,
  bottomtitle=2pt,
  before skip=4pt,
  after skip=8pt
}

\newtcolorbox{promptbox}[1][]{
  enhanced,
  breakable,
  title={#1},
  colback=PaleBlue,
  colframe=PreprintBlue,
  colbacktitle=PreprintBlue,
  coltitle=white,
  fonttitle=\bfseries\sffamily\footnotesize,
  fontupper=\ttfamily\footnotesize,
  boxrule=0.5pt,
  arc=2pt,
  left=6pt,
  right=6pt,
  top=4pt,
  bottom=4pt,
  toptitle=2pt,
  bottomtitle=2pt,
  before skip=5pt,
  after skip=6pt,
  before upper={\setlength{\parskip}{2pt}}
}

\newcommand{\method}{\textsc{SkillZip}\xspace}

\newcommand{\atoms}{\mathcal{A}}
\newcommand{\contract}{\mathcal{C}}
\newcommand{\library}{\mathcal{K}}
\newcommand{\residual}{\mathcal{R}}

\newcommand{\Description}[1]{}
\newcommand{\resulttakeaway}[1]{%
  \par\smallskip
  \noindent\textcolor{SkillZipBlue}{\textbf{Key result.}}~\textbf{#1}%
  \par\smallskip
}

\theoremstyle{definition}
\newtheorem{definition}{Definition}[section]

\theoremstyle{plain}
\newtheorem{proposition}[definition]{Proposition}
\newtheorem{corollary}[definition]{Corollary}

\lstdefinestyle{compactjson}{
  basicstyle=\ttfamily\scriptsize,
  columns=fullflexible,
  breaklines=true,
  frame=single,
  rulecolor=\color{Line},
  backgroundcolor=\color{PaleGray},
  xleftmargin=2pt,
  xrightmargin=2pt,
  aboveskip=5pt,
  belowskip=5pt,
  showstringspaces=false
}

\setlist[itemize]{leftmargin=*,topsep=2pt,itemsep=1pt,parsep=0pt,partopsep=0pt}
\setlist[enumerate]{leftmargin=*,topsep=2pt,itemsep=1pt,parsep=0pt,partopsep=0pt}
\newcommand{\AffiliationLogoStrip}{%
  \parbox{\textwidth}{%
    \normalfont\normalsize
    \raggedright
    \raisebox{0.010in}[0.33in][0pt]{%
      \includegraphics[height=0.255in,keepaspectratio]{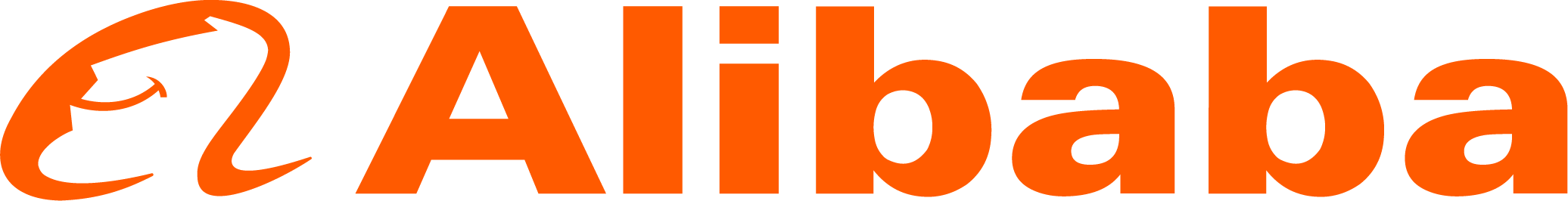}}%
    \hspace{0.30in}%
    \raisebox{-0.010in}[0.33in][0pt]{%
      \includegraphics[height=0.315in,keepaspectratio]{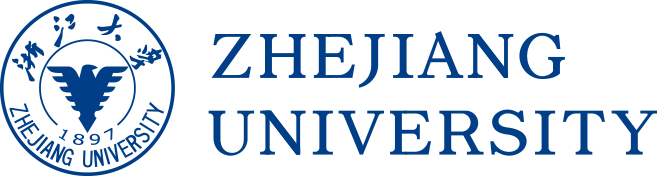}}%
    \hspace{0.30in}%
    \raisebox{0.005in}[0.33in][0pt]{%
      \includegraphics[height=0.285in,keepaspectratio]{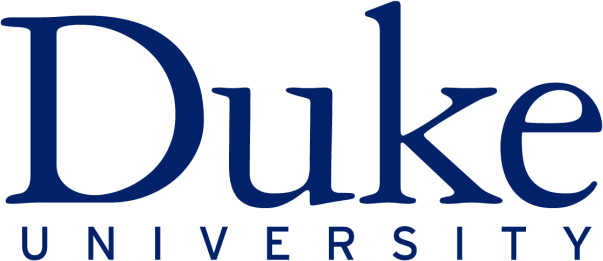}}%
    \par\vspace{0.105in}%
    \noindent\color{PreprintRule}\rule{\textwidth}{0.55pt}%
    \par\vspace{0.235in}%
  }%
}

\hypersetup{
  pdftitle={SkillZip: Evaluation-Free Skill Compression for Self-Evolving Agents by Discovering Reusable Structure},
  pdfauthor={Xiaofan Bai, Hongqiang Lin, Chao Liu, Yantao Zhang, Xuan Jin, Xipeng Cao, Yuhong Li},
  pdfsubject={Evaluation-free compression of self-evolving agent skills},
  pdfkeywords={LLM agents, self-evolving agents, agent skills, skill compression, minimum description length}
}

\title{%
  \AffiliationLogoStrip\par
  {\normalfont\fontsize{20.5}{24.0}\selectfont
  SkillZip: Evaluation-Free Skill Compression for\\[-0.12em]
  Self-Evolving Agents by Discovering Reusable Structure}
}

\author{%
  {\normalsize Xiaofan Bai$^{1,*}$, Hongqiang Lin$^{2,*}$, Chao Liu$^{1}$, Yantao Zhang$^{3}$,\\[-0.15em]
  Xuan Jin$^{1,\dagger}$, Xipeng Cao$^{1}$, and Yuhong Li$^{1}$}\\[-0.12em]
  {\footnotesize $^{1}$Alibaba Group \quad $^{2}$Zhejiang University \quad $^{3}$Duke University}\\[-0.14em]
  {\scriptsize\ttfamily
  baixiaofan.bxf@alibaba-inc.com, linhongqiang@zju.edu.cn
  }\\[-0.18em]
  
  {\scriptsize $^{*}$Equal contribution. \quad $^{\dagger}$Project leader.}%
}

\IEEEtitleabstractindextext{%
\begin{abstractbox}
\textbf{Abstract---}
Self-evolving agents accumulate reusable skills by appending successful procedures and failure fixes. Over time, the same requirement is often restated in several branches, examples, and warnings, while common action sequences are copied rather than reused. The resulting skill becomes expensive to inject and difficult to maintain. Generic prompt compression is ill-suited to this setting because a skill is not a flat passage: its name and description define when it applies, its workflow controls execution, its tool and output contracts constrain validity, and rare exceptions may remain essential even when no sampled task activates them. Evaluation-guided compression can test these behaviors, but it introduces rollouts, cost, and dependence on the compression-time evaluation set.

We present \method, an evaluation-free method that compresses a skill by finding its \emph{shortest faithful structural explanation}. The intuition is ``explain once, reference many'': state a repeated rule once at the scope where it applies, factor a repeated action sequence into a shared procedure, and keep only the differences as explicit exceptions. We formalize this intuition as a typed minimum-description-length objective over a skill contract and a residual, subject to a hard coverage constraint for every extracted trigger, workflow edge, tool requirement, obligation, and output field. The formulation provides simple sharing thresholds, preserves unique rare rules by construction, and supports efficient local updates. \method has a one-shot mode with one structured extraction call and deterministic optimization, and a continual \emph{Zip-on-Write} mode that integrates each self-evolution patch without replaying tasks or reparsing the full history. Through comprehensive experimental evaluations, we demonstrate the effectiveness and superiority of \method in compression performance, generalizability, and cost overhead.
\end{abstractbox}
\begin{projectbox}
\textbf{Project available at:}
\url{https://github.com/yutou520131/SkillZip}
\end{projectbox}
}

\begin{document}
\maketitle
\thispagestyle{empty}
\IEEEdisplaynontitleabstractindextext

\section{Introduction}
\label{sec:intro}

A self-evolving agent improves by turning experience into reusable instructions. When a tool fails, it appends a warning; when an answer violates a format, it adds an example; when a rare branch succeeds, it records the successful procedure. Each update is locally reasonable. The accumulated skill, however, is usually edited as an append-only notebook rather than maintained as a coherent program. After enough rounds, ``never overwrite the source file'' may appear in the introduction, three workflow branches, and an example, while the same validate--repair--verify sequence is copied repeatedly with only minor differences.

This creates a systematic mismatch between \emph{textual growth} and \emph{procedural growth}. New text keeps accumulating even after the number of genuinely new requirements begins to saturate. Because a loaded skill occupies the context on every invocation, redundant text increases prefill cost and can obscure the instructions that actually govern execution. Figure~\ref{fig:growth} illustrates the phenomenon that our longitudinal study will measure.

\begin{figure}[t]
  \centering
  \includegraphics[width=1.\columnwidth]{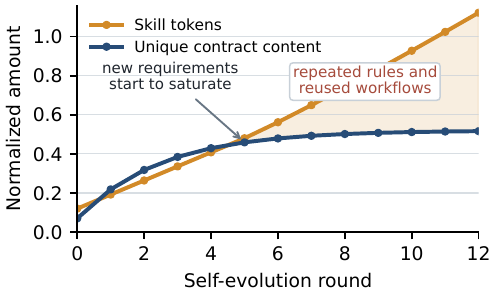}
  \caption{\textbf{The skill growth tendency with respect to self-evolving progress on diverse benchmarks for a CodeX Agent.} Self-evolution can keep increasing skill length after genuinely new procedural content has largely stabilized.}
  \Description{A line chart showing skill tokens continuing to grow while unique procedural content rises early and then saturates.}
  \label{fig:growth}
\end{figure}

Not all long skills are long for the same reason. A conventional one-shot or community-authored skill can mix operational rules with background exposition, templates, examples, or even content unrelated to execution; deleting or deferring such material is often appropriate~\cite{gao2026skillreducer}. A multi-round evolved skill is different. Recent systems accept bounded edits only after rollout feedback or validation, or aggregate recurrent successful and failed trajectories into updates~\cite{yang2026skillopt,ma2026skillclaw}. Its additional text is therefore usually more knowledge-dense: even a singleton warning may encode an expensive failure that the agent learned to avoid. The main redundancy is less often \emph{irrelevant topic} and more often \emph{repeated representation}---the same invariant copied into several branches, a workflow restated after each failure, or a general rule followed by progressively narrower exceptions. \textbf{This distinction changes the compression goal from content filtering to knowledge consolidating.}

The obvious response is prompt compression, but the abstraction is wrong. 
Prompt compressors typically decide which tokens are useful for a current query or likely answer~\cite{jiang2023llmlingua,jiang-etal-2024-longllmlingua,pan2024llmlingua2}, while skills should be reusable for all queries of certain tasks. 
More importantly, its meaning is not distributed uniformly across words. The name and description determine when the skill is selected; temporal phrases determine action order; tool arguments define valid calls; branch guards determine where a rule applies; and an output schema defines when the procedure is complete. A rare exception may be activated only once, yet deleting it can be more damaging than removing several paragraphs of rationale. Task-based validation is also not a complete solution. SkillReducer~\cite{gao2026skillreducer} demonstrates that structure-aware rewriting and progressive disclosure can reduce skill cost, and its feedback loop uses generated tasks to recover missed content. Such validation is valuable, but it makes compression expensive and couples the compressed results to compression-time tasks. A method can repeatedly repair the branches observed by that set while still losing an untested guard or output constraint. We therefore ask a stricter question: \emph{Can a skill be shortened using only the structure already present in the skill, without observing tasks, rewards, trajectories, or verifiers?}

Our answer begins with a simple observation: a skill resembles a compact operating manual more than a generic prompt. It contains an \textbf{interface} (name, purpose, triggers, exclusions), a \textbf{procedure} (steps, branches, loops, fallbacks), \textbf{contracts} for tools and outputs, and \textbf{rules} whose scope can be global or branch-specific. This structure reveals redundancy that token importance cannot see. A rule repeated in every branch can be stated once before the branch. A repeated action sequence can be named once and reused. Several guarded variants can be written as one common rule plus explicit exceptions. Conversely, a unique requirement cannot be removed merely because it is short or rare.

Motivated by this observation, the key compression principle of \method could be concluded as: \textbf{prefer the shortest faithful explanation of the skill}. Informally, the compressed representation pays once for every shared structure and pays separately only for genuine differences. This is the same intuition as replacing three copies of a code fragment by one function and three calls. Formally, it is an instance of minimum description length (MDL): minimize the cost of a compact skill contract plus the cost of the source details not explained by that contract, while requiring every extracted normative requirement to remain covered. We will further introduce the intuition before the formal objective in Section~\ref{sec:theory}.

To broaden the practical adaptability of \method, we design two modes.
\textbf{One-shot SkillZip} compresses an existing evolved skill checkpoint with one structured extraction call followed by deterministic optimization. \textbf{Zip-on-Write} participates in self-evolution: when a new self-evolution patch arrives, \method compares it only with compatible parts of the current compact contract, then either absorbs it, refines an existing rule, adds a new requirement, or triggers a local refactoring. Periodic repacking captures patterns that become reusable only after several updates.

\begin{glancebox}
\textbf{Compress once:} extract a typed contract, consolidate reusable rules and workflows, and render a shorter skill while preserving every covered requirement. \textbf{Maintain continually:} use Zip-on-Write to absorb each validated patch through local, scope-aware updates. \textbf{Evidence contract:} compression observes no benchmark tasks, rollouts, rewards, or behavioral verifier.
\end{glancebox}

Our main contributions are:
\begin{itemize}[leftmargin=*,topsep=2pt,itemsep=1pt]
  \item We formulate \textbf{evaluation-free skill compression} as contract representation tailored for evolved skills.
  \item We derive a \textbf{shortest faithful explanation objective} that unifies semantic sharing, scope lifting, workflow reuse, and exception encoding. A hard coverage constraint gives a rare-rule preservation guarantee independent of task frequency.
  \item We design \textbf{one-shot SkillZip} and \textbf{Zip-on-Write}. The first uses one structured extraction call followed by deterministic optimization; the second maintains a compact state during self-evolution and avoids replaying tasks or reparsing the full history.
  \item We perform comprehensive empirical evaluations showing that SkillZip delivers substantial gains in compression performance, robust generalizability, and low cost overhead.
\end{itemize}
\section{Related Work}
\label{sec:related}

\subsection{Self-Evolving Agents and Persistent Skills}
Agents increasingly convert interaction experience into reflections, memories, programs, or reusable skills~\cite{wang2023voyager,shinn2023reflexion,gao2025selfevolving}. Systems such as ACE, SkillRL, and SkillClaw explicitly maintain and evolve persistent skill artifacts across tasks or users~\cite{zhang2025ace,xia2026skillrl,ma2026skillclaw}; SkillRevise and SkillGrad use execution traces to diagnose and improve an existing skill~\cite{liu2026skillrevise,wang2026skillgrad}. These works focus on how procedural knowledge is acquired. As the artifact grows, however, acquisition and maintenance become different problems: a useful new patch can still duplicate an old invariant or repeat an existing workflow. \method addresses this consolidation problem without replaying the experiences that produced the skill.

\subsection{Prompt and Context Compression}
Prompt compression prunes or rewrites context using perplexity, query relevance, learned token labels, summarization, or latent representations~\cite{li2023selective,jiang2023llmlingua,jiang-etal-2024-longllmlingua,pan2024llmlingua2,mu2023gisting,xu2024recomp}. These techniques usually treat input as a sequence whose importance is estimated relative to a query, answer, or representative task distribution. Skills violate both assumptions: future tasks are unknown, and procedural meaning depends on typed relations such as ``before'', branch guards, tool arguments, and required output fields. Query-conditioned compression may be effective for a single request, but a reusable skill must retain requirements that no compression-time query activates. \textbf{\method therefore compresses repeated procedural structure rather than low-salience tokens.}

\subsection{Skill Compression and Efficient Execution}
SkillReducer is the closest textual baseline. It minimizes routing descriptions with delta debugging, classifies body content, moves supplementary material to on-demand references, and uses faithfulness checks and task-based feedback~\cite{gao2026skillreducer}. It establishes that skill content should not be treated as homogeneous text and is especially well matched to ordinary skills that contain background, examples, and reference material. Multi-round evolved skills present a complementary regime: accepted updates are often operationally relevant, while redundancy appears as repeated constraints, overlapping workflows, and accumulated exceptions. In addition, SkillReducer's generated tasks and feedback loop participate directly in candidate selection. Repeatedly tuning an artifact on a finite evaluation sample can create sample-specific selection effects, a general issue in adaptive data analysis~\cite{dwork2015adaptive}. \method removes this source of dependence by never exposing compression to tasks, and it keeps a single portable text skill rather than changing the loading architecture.

SKIM and TokMem encode procedural knowledge into learned soft or memory tokens~\cite{wang2026skim,wu2026tokmem}, while Skill-to-LoRA transfers a textual skill into model parameters~\cite{zhang2026skilltolora}. Their representations can be compact, but they are model-dependent and less convenient to inspect, diff, or update during continual evolution. \method uses a structured sidecar only while maintaining the artifact and renders an ordinary human-readable skill for deployment. Skill~\cite{zhang2026formalskill} and SkillRT~\cite{chen2026skillrt} treat skills as executable or compilable artifacts. Their goals are reliable execution and portability rather than shortening an accumulated natural-language skill. Nevertheless, they support the premise underlying our representation: \emph{a skill has an interface, control flow, and contracts, not merely topical text}.

\subsection{Minimum Description Length}
The minimum description length (MDL) principle selects the model that gives the shortest joint description of a model and the data it explains~\cite{grunwald2007mdl,galbrun2022mdl}. Grammar-based methods such as SEQUITUR and Re-Pair similarly replace repeated subsequences with reusable rules~\cite{nevill1997sequitur,larsson2000repair}. We adopt such an intuitive principle---\emph{a shared explanation is useful when defining it once and referencing it is cheaper than repeating it}---and adapt it to typed procedural knowledge. Unlike sequence compression, our objective may not merge identical-looking clauses if they occur under incompatible guards, and it may not delete a unique rule even when that rule has zero repetition. The hard coverage is therefore as important as the skill length.

\section{From Skill Text to a Compact Contract}
\label{sec:structure}

\subsection{Evaluation-Free Compression}
Let $S$ be one textual skill and $\widetilde S$ its compressed form. During compression, a method may read $S$, files referenced by $S$, and---in the continual setting---the sequence of skill patches. It may not access downstream tasks, execution trajectories, rewards, or behavioral verifiers. These resources are used only after compression to measure generalization. We seek three properties. First, $\widetilde S$ should be substantially shorter under the deployment tokenizer. Second, it should preserve what the skill \emph{requires}, including rare conditions. Third, it should remain an ordinary text artifact that can be inspected, versioned, and used by different agent backbones.

The second requirement is deliberately structural. Without running tasks, we cannot prove that arbitrary natural language induces identical model behavior. We can, however, require the compressor to preserve every operational element it extracts from the source, and to retain ambiguous source spans verbatim. This makes the boundary of the guarantee explicit rather than hiding it behind a finite evaluation suite.

\subsection{The Contract Hidden Inside a Skill}
Figure~\ref{fig:anatomy} shows the representation used by \method. We write the extracted contract as
\begin{equation}
\contract(S)=\langle I,G,T,C,O,E\rangle,
\label{eq:skill_tuple}
\end{equation}
where:
\begin{itemize}[leftmargin=*,topsep=2pt,itemsep=1pt]
  \item $I$ is the \textbf{interface}: skill name, purpose, positive triggers, and exclusions;
  \item $G$ is the \textbf{workflow}: actions, order, decisions, loops, fallbacks, and stop conditions;
  \item $T$ is the \textbf{tool protocol}: tool names, required arguments, preconditions, expected observations, and error handling;
  \item $C$ is the set of \textbf{scoped rules}: what the agent must, must not, or preferably should do, together with the scope and guard under which each rule applies;
  \item $O$ is the \textbf{output contract}: response type, required fields, ordering, validation, and completion conditions; and
  \item $E$ is \textbf{supporting evidence}: examples, templates, and rationale linked to the contract elements they express.
\end{itemize}

\begin{figure}[t]
  \centering
  \includegraphics[width=0.85\columnwidth]{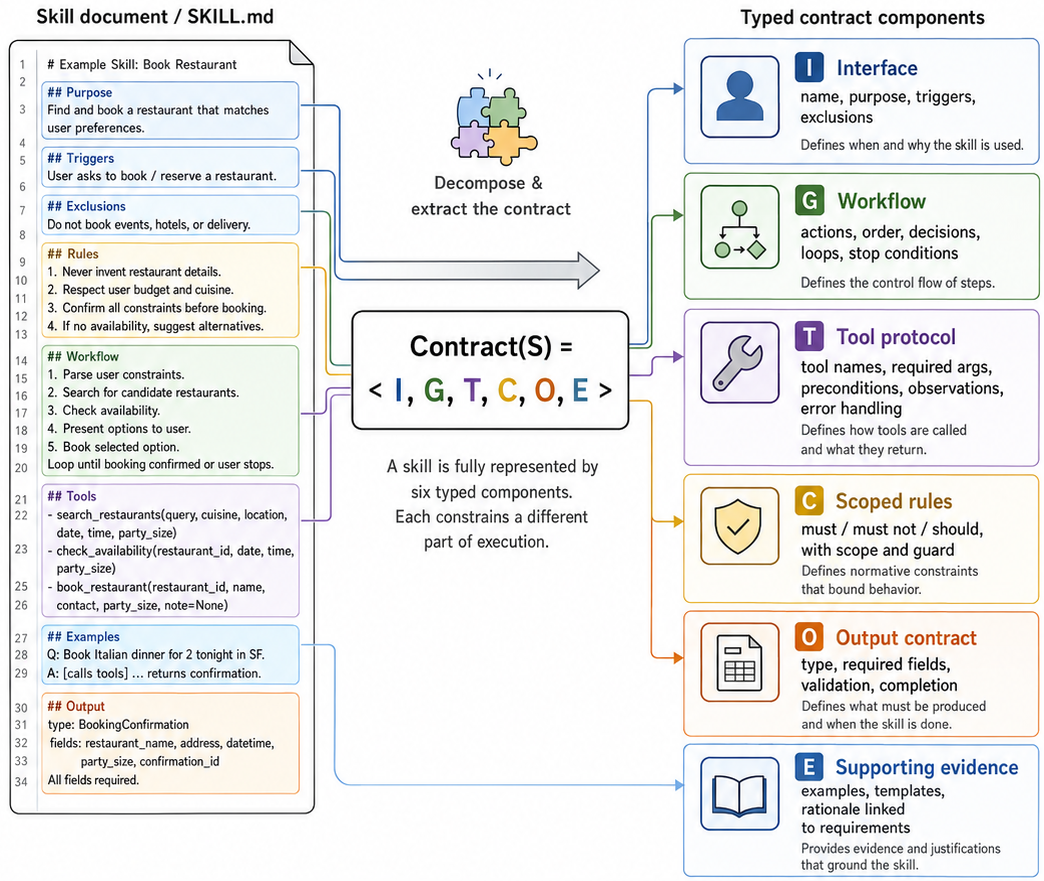}
  \caption{\textbf{A skill is a typed contract.} Different text spans constrain different parts of execution. An example is removable only when every requirement it uniquely expresses is represented elsewhere.}
  \Description{A skill document is mapped to an interface, control-flow graph, tool contracts, scoped obligations, output schema, and residual evidence.}
  \label{fig:anatomy}
\end{figure}

This decomposition is not cosmetic. It changes which compressions are safe. Two sentences about the same tool cannot merge if they require different arguments. A prohibition repeated in all branches may move to the common parent scope, but a prohibition attached to only one guarded branch may not. Two examples can be removed if they merely illustrate an explicit output schema; an example that is the sole source of a required field must first be converted into an explicit output rule.

\subsection{Typed Units, Scope, and Coverage}
The parser maps source spans to typed units $a\in\atoms$. A unit records
\begin{equation}
a=(\tau,\sigma,g,m,p,\mathcal{P}),
\end{equation}
where $\tau$ is its type, $\sigma$ its scope, $g$ an optional guard, $m$ its modality, $p$ its normalized content, and $\mathcal{P}$ the supporting source spans. A workflow unit additionally records incoming and outgoing edges; a tool unit records its argument signature; an output unit records field and validation information.

If a source span cannot be interpreted with sufficient confidence, it becomes a \emph{locked residual}: it is copied verbatim and excluded from deletion. This conservative fallback is important in practice because the structural parser, rather than the later optimizer, is the main source of semantic uncertainty.

We write $a\preceq K$ when compact representation $K$ covers unit $a$. Coverage can be direct, or structural: a branch-local copy can be covered by the same rule placed at the closest common ancestor; a repeated action sequence can be covered by a shared procedure whose expansion contains the original edges. Coverage is type-sensitive. A tool name does not cover its required arguments, and a general rule does not cover a conflicting guarded exception.

\begin{definition}[Contract coverage]
A compact representation $K$ covers a parsed skill if
\begin{equation}
\forall a\in\atoms_{\mathrm{req}}(S),\qquad a\preceq K,
\label{eq:coverage}
\end{equation}
where $\atoms_{\mathrm{req}}$ contains interface conditions, workflow nodes and edges, tool requirements, scoped rules, and output requirements.
\end{definition}

The next section asks which covering representation should be selected. The answer is not to merge everything similar, but to choose the representation that is shortest after accounting for definitions, references, exceptions, and unexplained residual text.

\section{Theoretical Analysis}
\label{sec:theory}

Section~\ref{sec:structure} defines the \emph{contract} that must survive
compression. The remaining question is how to represent that contract
compactly. Our intuition is the same as refactoring repeated code:
\emph{state shared structure once, reference it where needed, and keep only
genuine differences explicit}. A shared abstraction is useful only when this
representation is shorter than repeating all of its instances.

\subsection{The Shortest Faithful Explanation}

We represent a compressed skill by a library $\library$ of reusable contract
elements and a residual $\residual$ for unique, exceptional, or uncertain
content. The library may contain shared rules, workflow fragments, tool
contracts, output fields, and interface entries. The residual preserves
explicit exceptions that cannot be safely normalized.

SkillZip selects the shortest representation that still covers every required
contract unit:
\begin{equation}
\begin{aligned}
(\library^*,\residual^*)
&=
\arg\min_{(\library,\residual)\in\mathcal{H}(S)}
\Big[
L(\library)+L(\residual\mid\library)
\Big] \\
\text{s.t.}\quad
&a\preceq(\library,\residual),
\qquad
\forall a\in\atoms_{\mathrm{req}}(S).
\end{aligned}
\label{eq:objective}
\end{equation}
Here, $\mathcal{H}(S)$ is the finite set of candidate representations proposed
from the parsed skill, and $L(\cdot)$ measures the rendered token cost,
including the overhead of definitions, references, and scope annotations.
Equation~\eqref{eq:objective} is an instance of minimum description length,
but its operational meaning is simple: \textbf{compression may change how a
requirement is written, but not whether it remains represented}. In
particular, a unique requirement cannot be removed merely as it is short or unsupported by frequently sampled tasks.

\subsection{One Objective, Four Skill-Specific Decisions}

The same objective governs four forms of reusable structure.

\begin{itemize}
    \item \textbf{Equivalent requirements:} paraphrases are represented once
    when a shared rule plus any residual differences is shorter than keeping
    separate copies.

    \item \textbf{Repeated rules across scopes:} a rule is moved to the nearest
    common scope only when it applies to every relevant path; conflicting
    branch-local behavior remains as an explicit exception.

    \item \textbf{Repeated workflows:} a recurring action sequence becomes a
    shared procedure only when the saved repetitions exceed the cost of
    defining and calling it.

    \item \textbf{Guarded variants:} related clauses may be represented as one
    common rule plus guarded deltas, but only when the exception structure is
    both faithful and shorter than listing the variants separately.
\end{itemize}

These decisions are therefore not independent rewrite heuristics. They are
alternative ways of covering the same typed contract, compared under one
length objective. The explicit cost tests and scope conditions are given in
Appendix~\ref{app:theory-details}.

\subsection{Preservation and Continual Compression}

\begin{proposition}[Parsed-contract preservation]
\label{prop:coverage}
If every normative source span is represented by a typed unit or residual, any feasible solution of Eq.~\eqref{eq:objective} preserves
all extracted requirements.
\end{proposition}

Prop.~\ref{prop:coverage} follows directly from the hard coverage constraint. Its most
important consequence is independent of the length model:

\begin{corollary}[Rare-rule preservation]
\label{cor:rare}
The preservation of a unique requirement does not depend on how often its
branch appears in any compression-time task distribution.
\end{corollary}

This is the key theoretical benefit of evaluation-free compression. A guard,
tool argument, exception, or output field is protected because it belongs to
the parsed contract, not because sampled tasks happen to activate it. The
guarantee is intentionally limited to the extracted contract; uncertain spans
are therefore kept verbatim rather than silently discarded. The objective also yields an efficient continual form. A new patch usually
changes only a small type--scope neighborhood of the current compact contract.
If the patch creates no profitable abstraction that crosses this neighborhood,
all other cost terms remain constant, so local optimization gives the same
update as rerunning the batch objective. When several patches collectively
create new cross-scope reuse, occasional global repacking restores the missed
saving. Thus local updates provide efficiency, while repacking recovers
long-range structure. Formal conditions and proofs are deferred to
Appendix~\ref{app:theory-details}.

\section{SkillZip}
\label{sec:method}

Figure~\ref{fig:method} summarizes the method. One-shot compression converts a skill into a compact contract and then renders it back to text. Zip-on-Write maintains the same contract while the agent evolves, so each new patch is consolidated before it becomes permanent.

\begin{figure*}[t]
  \centering
  \includegraphics[width=.9\textwidth]{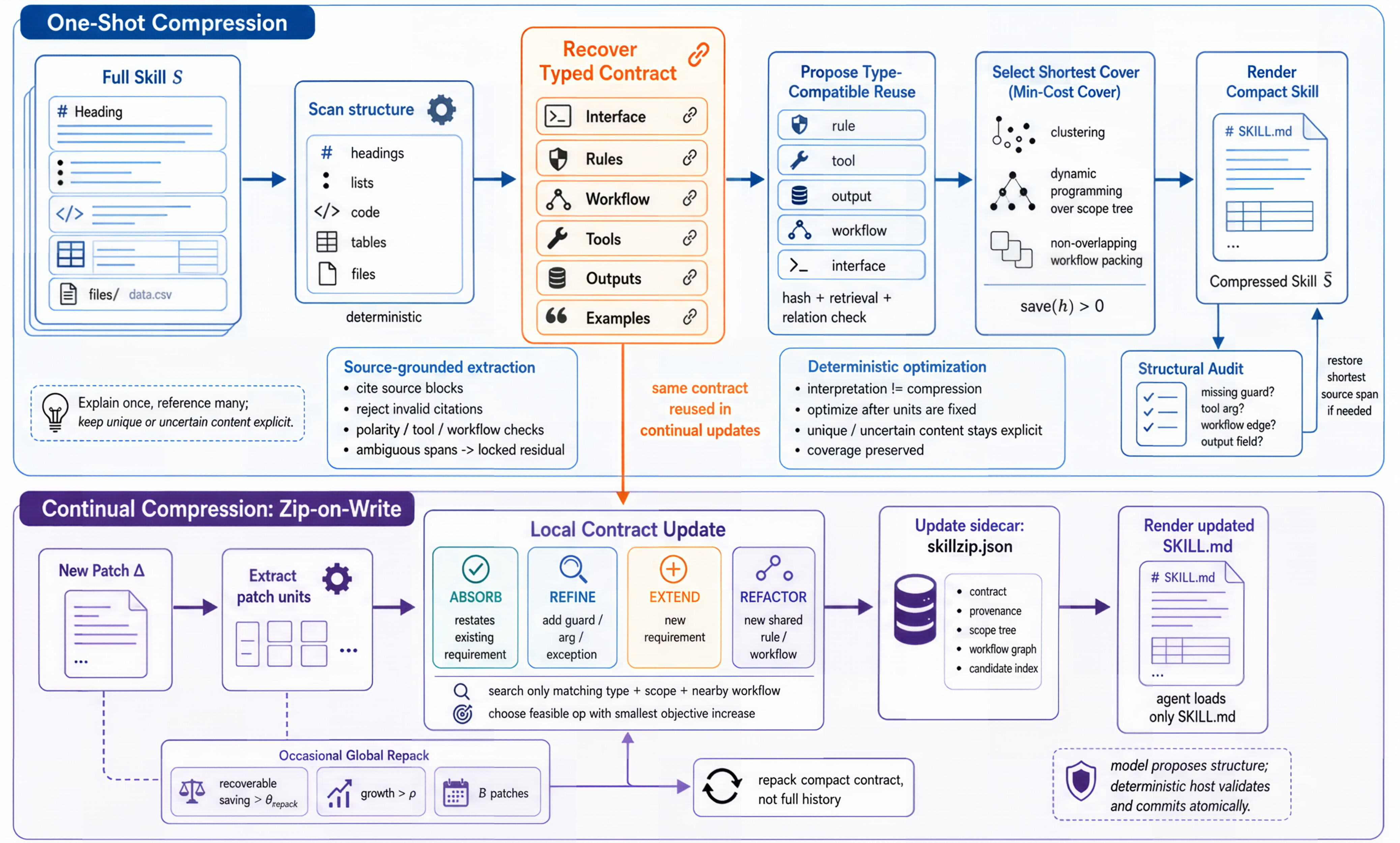}
  \caption{\textbf{Overview of \method.} One-shot compression first recovers the skill contract, then applies the ``explain once, reference many'' principle to repeated rules and workflows, while unique and uncertain content remains explicit. Zip-on-Write compares each patch with the affected contract neighborhood and performs occasional repacking when reuse accumulates across patches.}
  \Description{A two-lane diagram. The one-shot lane maps a long skill to a typed contract, consolidates repeated structure, and renders a compact skill. The continual lane integrates patches through local contract updates and occasional global repacking.}
  \label{fig:method}

\end{figure*}

\subsection{One-Shot Compression}

\paragraph{1. Scan the SKILL.MD before using a model.}
A deterministic scanner parses front matter, headings, nested lists, code blocks, tables, and file references. It copies the skill name and description as interface candidates, converts Markdown nesting into a preliminary scope tree, and records stable source-block identifiers. Numbered lists and temporal markers provide high-confidence workflow hints. This pass reduces the amount of structure that the language model must infer and makes every later compression decision traceable to source skill text.

\paragraph{2. Recover the typed contract once.}
A schema-constrained model receives the numbered blocks and returns the contract in Eq.~\eqref{eq:skill_tuple}: interface entries, workflow nodes and edges, tool calls and required arguments, scoped rules with modality and guard, output fields, and links from examples to the requirements they specify. Every extracted unit must cite source blocks. The host rejects unsupported citations, polarity mismatches, unknown tool names, and invalid workflow references. Ambiguous spans are placed in the locked residual. The extractor is deliberately not asked to compress. Separating interpretation from optimization has two benefits: contract recovery can be evaluated against human annotations, and the optimization is deterministic once the extracted units are fixed.

\paragraph{3. Propose only type-compatible reuse.}
Candidate generation uses hard structural blocking: (a) interface entries compare with the same trigger or exclusion role; (b) rules compare when modality, predicate family, and scope ancestry are compatible; (c) tool units require the same tool and compatible argument signatures; (d) output units compare within the same response type and field namespace; (e) and workflow reuse is proposed from repeated guarded action sequences with identical entry and exit behavior.
Exact matches are found by hashing. Near duplicates are retrieved with an embedding index and verified by a frozen relation checker that predicts equivalence, implication, conflict, or unrelatedness. Conflicts create exception candidates rather than merge candidates.

\paragraph{4. Select the shortest covering explanation.}
Each candidate $h$ receives a saving
\begin{equation}
\operatorname{save}(h)=L(\text{separate form})-L(\text{form using }h).
\end{equation}
Non-positive candidates are discarded. Equivalent units are clustered after conflict filtering. Rule placement is solved by dynamic programming over the scope tree. Repeated workflow candidates are selected by non-overlapping weighted packing, using saving per covered token as the greedy order and pairwise exchange as a refinement. After each selection, the optimizer checks that all source units remain covered. This step is where the running example is simplified. The two branch-local copies of ``never overwrite the input'' are represented by one rule at their common parent because Eq.~\eqref{eq:scope_lift} is satisfied. The validate step is shared only if its definition and calls are shorter than the copies. The JSON example is removed only after its fields are covered by the output contract.

\paragraph{5. Render with fixed templates and audit structurally.}
The renderer produces a normal skill with a concise purpose and triggers, global rules, a numbered workflow, nested guarded branches, explicit tool requirements, and an output checklist or schema. A shared workflow is named only when references save tokens; otherwise it remains inline. Exceptions are placed after their base rule. An optional structural audit reparses only the compressed skill and compares it with the selected contract. If a trigger, guard, workflow edge, tool argument, polarity, or output field is missing, \method restores the shortest source span that covers it and locks that span. 

\begin{algorithm}[t]
\caption{One-Shot \method}
\label{alg:oneshot}
\begin{algorithmic}[1]
\Require Skill $S$ and audit flag $v$
\Ensure Compressed skill $\widetilde S$
\State $B\gets\Call{Scan}{S}$
\State $(\atoms,R)\gets\Call{ExtractContract}{B}$
\State $H\gets\Call{ProposeReuse}{\atoms}$
\State $K\gets\Call{MinCostCover}{\atoms,R,H}$
\State $\widetilde S\gets\Call{Render}{K}$
\If{$v$}
  \State $\widehat K\gets\Call{Parse}{\widetilde S}$
  \State $M\gets\Call{ContractDiff}{K,\widehat K}$
  \State $\widetilde S\gets\Call{RestoreMissingSpans}{\widetilde S,M}$
\EndIf
\State \Return $\widetilde S$
\end{algorithmic}

\end{algorithm}

\subsection{Continual Compression: Zip-on-Write}
A self-evolving agent usually produces a small patch rather than a complete rewrite. \method stores a sidecar, \texttt{skillzip.json}, containing the current contract, source provenance, scope tree, workflow graph, and candidate indices. The rendered \texttt{SKILL.MD} remains the only artifact loaded by the agent.

For every patch unit, the updater compares four interpretations under the same objective:
\begin{itemize}[leftmargin=*,topsep=2pt,itemsep=1pt]
  \item \texttt{ABSORB}: the patch restates an existing requirement and adds no new contract content;
  \item \texttt{REFINE}: the patch adds a guard, tool argument, validation, or explicit exception to an existing unit;
  \item \texttt{EXTEND}: the patch introduces a genuinely new requirement; and
  \item \texttt{REFACTOR}: the patch makes a shared rule or workflow newly worthwhile.
\end{itemize}
The host selects the feasible operation with the smallest increase in Eq.~\eqref{eq:objective}. Candidate search is restricted to the matching type, current scope, ancestor scopes, and adjacent workflow nodes. Thus a patch with $d$ extracted units compares against $O(dk)$ retrieved candidates rather than the complete history.

Local updates may miss reuse that becomes profitable only after several patches. The sidecar therefore tracks approximate counts for rule families and action $n$-grams. A global repack is triggered when the estimated recoverable saving exceeds $\theta_{\mathrm{repack}}$, the contract grows by more than $\rho$ since the last repack, or $B$ patches have arrived. Repacking operates on the compact contract rather than all historical prose.

Algorithm~\ref{alg:ziponwrite} in Appendix~\ref{app:theory-details} summarizes the continual path explicitly. Crucially, $\Delta_t$ is frozen before compression: the evolver decides what knowledge is learned, while \method decides only how that knowledge is represented. \textsc{Absorb} is feasible only when the patch adds no uncovered contract unit; \textsc{Refine} preserves the old unit and records its new guard, argument, validation, or exception; \textsc{Extend} adds a new required unit; and \textsc{Refactor} changes representation without changing coverage. No operation is accepted because it improves a task score.
Appendix~\ref{app:implementation} specifies the schema, cache keys, candidate algorithms, prompts, transaction protocol, and recovery path needed for a reproducible implementation.

\paragraph{Compression as a skill.}
Skill compression can also be packaged as a skill. The model receives the new patch and a retrieved slice of the sidecar, then proposes one of the four operations with source citations. A deterministic host validates the schema, recomputes the saving, enforces coverage, writes a transaction log, and atomically replaces the skill only after rendering succeeds. The model proposes structure but never directly mutates persistent state.

%
%
\providecolor{SkillZipBlue}{RGB}{35,79,119}
\providecolor{SkillZipGreen}{RGB}{45,105,78}
\providecolor{SkillZipPaleBlue}{RGB}{232,240,247}
\providecolor{SkillZipPaleGray}{RGB}{245,246,247}
\providecommand{\resulttakeaway}[1]{%
  \par\smallskip
  \noindent\textcolor{SkillZipBlue}{\textbf{Key result.}}~\textbf{#1}
  \par\smallskip
}

\section{Experiments}
\label{sec:experiments}

We evaluate \method along the two claims that motivate its design:
self-evolution introduces substantially more \emph{surface text} than new
procedural knowledge, and this redundancy can be removed without consulting
downstream evaluations. The section first defines the controlled and
model-backed protocols, then answers five research questions:
\textbf{RQ1} quantifies skill growth during self-evolution;
\textbf{RQ2} studies the compression--fidelity trade-off;
\textbf{RQ3} measures compression cost;
\textbf{RQ4} evaluates the compressed skills' generalizability of \method against the baseline; and
\textbf{RQ5} evaluates continual Zip-on-Write compression.

\subsection{Experimental Setup}
\label{sec:experimental_setup}

\paragraph{Models and Benchmarks.}
We evaluate three agent model backbones:
\textsc{Qwen3.7-Max}, \textsc{Qwen3.6-Plus}, and
\textsc{Kimi K2.6}
~\cite{qwen2026qwen37,qwen2026qwen36,moonshot2026kimik26}.
The two Qwen models represent different capability tiers within the same model
family, while Kimi K2.6 provides a cross-family evaluation. For every
model--benchmark pair, all skill conditions use the same model snapshot, agent
scaffold, system prompt, tool definitions, maximum interaction budget, and
decoding configuration. 
We consider three benchmarks with complementary procedural requirements.
\textbf{BFCL-v4 Web Search}
~\cite{patil2025bfcl,bfcl2025websearch}
evaluates multi-step web retrieval and reasoning through standardized search
and webpage-access tools.
\textbf{LiveMathematicianBench}
~\cite{he2026livemathematicianbench}
contains theorem-grounded multiple-choice problems derived from recent
mathematical research and tests precise reasoning about assumptions,
quantifiers, equivalence relations, and boundary conditions.
\textbf{SpreadsheetBench}
~\cite{ma2024spreadsheetbench}
evaluates spreadsheet manipulation on realistic user requests and workbook
files. 

\paragraph{Skill construction.}
For each model--benchmark pair, we begin with a manually authored
\emph{human skill} that specifies the task procedure and output requirements.
We then apply \textsc{SkillOpt}
~\cite{yang2026skillopt}
to improve this seed skill using the designated evolution split. SkillOpt
iteratively converts execution feedback into bounded edits of the skill
document. The resulting skill is frozen after evolution and used as the common
input to all compression methods. The evolution split is disjoint from the final benchmark test set. Thus, test
instances cannot affect either the knowledge acquired during skill evolution
or the subsequent compression decisions. Note that for one-shot SkillZip, we use a single fixed compressor model (Qwen3.7-max) across all skills, whereas for continual Zip-on-Write the agent's backbone model compresses its own skill. In both modes, the model is invoked only for schema-constrained contract extraction and relation/merge adjudication (greedy decoding, temperature 0); the minimum-cost covering and rendering are deterministic.

\paragraph{Baselines.}
We compare SkillZip with the following five Baselines: (1) \textbf{No Skill}: the backbone model is evaluated without any task-specific skill; (2) \textbf{Human Skill}: the original manually authored seed skill is provided to the model; (3)\textbf{Evolved Skill}: the complete, uncompressed skill produced by SkillOpt is provided to the model; (4)\textbf{SkillReducer}: SkillReducer~\cite{gao2026skillreducer} is applied once to the frozen evolved skill. The uncompressed evolved skill serves as the fidelity reference for the two
compression methods.  SkillReducer is the primary compression baseline because
it is the closest prior method that directly optimizes textual agent skills,
rather than generic prompts or interaction histories. We run SkillReducer
using its native structure-aware pipeline. SkillZip receives
exactly the same evolved skill but does not access benchmark tasks, execution
trajectories, rewards, or behavioral verifiers during compression.

\subsection{RQ1: Skill Growth in Self-Evolution?}

Figure~\ref{fig:skill_growth} shows that skill length increases monotonically
with self-evolution rounds across all benchmarks. By Round~5, the skills reach
approximately $5.6\times$, $3.1\times$, and $6.7\times$ their initial sizes on
BFCL-V4, LiveMath, and SpreadsheetBench, respectively, with an average growth
of about $5.2\times$. The growth persists across domains and is especially pronounced for tasks that
continually accumulate tool-use procedures, failure corrections, and output
constraints. Although each update may be locally useful, repeated rules,
overlapping workflows, and increasingly specific exceptions accumulate without
global consolidation, producing substantial \emph{skill bloat} and increasing
the context cost of every subsequent invocation.


\begin{figure}[htbp]
  \centering
  \includegraphics[width=0.9\columnwidth]{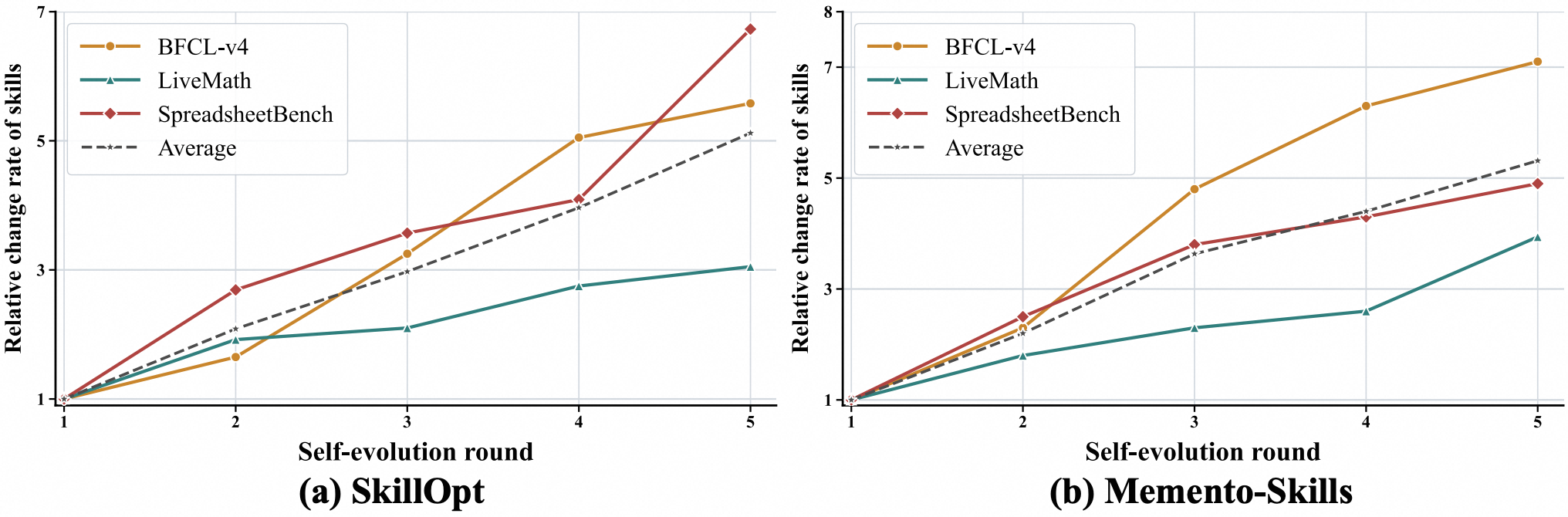}
  \caption{\textbf{The skill growth tendency with respect to agent self-evolving progress for SkillOpt\cite{yang2026skillopt} and Memento-Skills\cite{zhou2026mementoskillsletagentsdesign}.}}
  \Description{A line chart showing skill tokens continuing to grow while unique procedural content rises early and then saturates.}
  \label{fig:skill_growth}
\end{figure}

\begin{takeawaybox}
Self-evolution expands skills to more than
\textbf{$5\times$ their initial length on average}, making compression necessary to prevent procedural knowledge from becoming an
increasing context burden.
\end{takeawaybox}

\subsection{RQ2: Can SkillZip Preserve Skill Fidelity?}
\label{sec:behavioral_results}


Tab.~\ref{tab:cross_model_results} compares the uncompressed SkillOpt-evolved
skills with their compressed ones and other skill settings. The evolved skill outperforms the
no-skill and human-skill conditions in eight of nine settings, confirming that
multi-round evolution generally accumulates useful procedural knowledge.
\method retains this knowledge while achieving compression rates of
27.1\%--36.9\% (\textbf{31.2\% on average}). Its macro-average score is 0.577,
slightly exceeding the uncompressed evolved skill at 0.570, and it matches or
improves the evolved skill in five of nine settings. These results suggest that
consolidating repeated rules, scopes, and workflows can reduce skill length
without systematic behavioral degradation, even though \method uses no tasks,
rollouts, or verifiers during compression.

\begin{table}[htbp]
\centering
\caption{Compressed skill performance and average skill compression rate. The best task performance is shown in \textbf{bold}. Compression rate (C-Rate) is measured relative to the corresponding uncompressed evolved skill.}
\label{tab:cross_model_results}
\scriptsize
\setlength{\tabcolsep}{2.5pt}
\renewcommand{\arraystretch}{1.08}
\resizebox{\columnwidth}{!}{%
\begin{tabular}{cccccc}
\toprule
Model
& Skill Condition
& BFCL-V4 $\uparrow$
& LiveMath $\uparrow$
& Spreadsheet $\uparrow$
& \shortstack{Avg. C-Rate$\uparrow$} \\
\midrule

\multirow{5}{*}{Qwen-3.7-Max}
& No Skill      & 0.833 & 0.425 & 0.432 & -- \\
& Human Skill   & 0.848 & 0.396 & 0.436 & -- \\
& Evolved Skill & \textbf{0.869} & \textbf{0.474} & 0.525 & 0\%  \\
& SkillReducer  & 0.828 & 0.428 & \textbf{0.538} & 10.5\% \\
\rowcolor{SkillZipPaleBlue}
& \textbf{\method} & 0.863 & 0.472 & 0.519 & \textbf{27.1\%} \\
\midrule

\multirow{5}{*}{Qwen-3.6-Plus}
& No Skill      & 0.623 & 0.417 & 0.379 & -- \\
& Human Skill   & 0.641 & 0.405 & 0.457 & -- \\
& Evolved Skill & 0.685 & 0.392 & 0.472 & 0\%  \\
& SkillReducer  & 0.621 & 0.385 & 0.484 & 3.6\% \\
\rowcolor{SkillZipPaleBlue}
& \textbf{\method}
& \textbf{0.694}
& \textbf{0.435}
& \textbf{0.491}
& \textbf{29.7\%} \\
\midrule

\multirow{5}{*}{Kimi-K2.6}
& No Skill      & 0.714 & 0.402 & 0.415 & -- \\
& Human Skill   & 0.708 & 0.362 & 0.473 & -- \\
& Evolved Skill & \textbf{0.772} & 0.433 & 0.506 & 0\% \\
& SkillReducer  & 0.732 & 0.384 & 0.497 & 13.4\% \\
\rowcolor{SkillZipPaleBlue}
& \textbf{\method}
& 0.747
& \textbf{0.457}
& \textbf{0.513}
& \textbf{36.9\%} \\
\bottomrule
\end{tabular}%
}
\end{table}

Compared with SkillReducer~\cite{gao2026skillreducer}, \method achieves both
higher compression (31.2\% vs.\ 9.2\% on average) and higher task performance
(0.577 vs.\ 0.544). 
\emph{This difference
reflects the distinct compression regimes targeted by the two methods. SkillReducer is well suited to an initial debloating and quality-control pass
over heterogeneous public skills, which may contain verbose background,
redundant examples, or content unrelated to execution. }In contrast,
self-evolved skills are typically knowledge-dense because their updates arise
from execution feedback; their main redundancy lies in repeatedly stated
constraints, overlapping scopes, and copied workflows. Consequently,
\textbf{structural consolidation is better aligned with evolved-skill
compression than content filtering or deferral}.

\begin{takeawaybox}
\method compresses evolved skills by \textbf{31.2\% on
average} while preserving or improving their overall performance; SkillReducer
is more naturally positioned as a first-pass debloating and quality-control
method for general public skills.
\end{takeawaybox}

\begin{table}[htbp]
\centering
\caption{Compression overhead of SkillZip and SkillReducer.}
\label{tab:compression-efficiency}
\scriptsize
\setlength{\tabcolsep}{3.0pt}
\renewcommand{\arraystretch}{1.10}
\resizebox{0.95\columnwidth}{!}{%
\begin{tabular}{ccccc}
\toprule
Method
& Dataset
&Average Time $\downarrow$
& \shortstack{LLM Calls $\downarrow$}
& \shortstack{Rollouts $\downarrow$} \\
\midrule

\rowcolor{SkillZipPaleBlue}
\textbf{\method}
& LiveMath
& \textbf{207 s}
& 4
& \textbf{0} \\
SkillReducer
& LiveMath
& 1331 s
& \textbf{3}
& 40 \\
\midrule

\rowcolor{SkillZipPaleBlue}
\textbf{\method}
& Spreadsheet
& \textbf{332 s}
& 5
& \textbf{0} \\
SkillReducer
& Spreadsheet
& 1082 s
& \textbf{3}
& 80 \\
\midrule

\rowcolor{SkillZipPaleBlue}
\textbf{\method}
& BFCL-V4
& \textbf{318 s}
& 8
& \textbf{0} \\
SkillReducer
& BFCL-V4
& 587 s
& \textbf{3}
& 40 \\
\bottomrule
\end{tabular}%
}
\par\vspace{2pt}
\parbox{0.94\columnwidth}{\scriptsize\emph{Note:} ``LLM calls'' counts only calls to the compressor model. SkillReducer additionally consumes 40--80 validation rollouts per compression (each requiring at least one agent call). Because some rollouts hit a warm evaluation cache, the reported time is an optimistic lower bound.}
\end{table}

\subsection{RQ3: Compression Efficiency}
\label{sec:compression-efficiency}
To assess the compression cost of \method compared to SkillReducer, we compute the offline one-shot cost with SkillReducer on all benchmarks under the same execution environment. \method completes
compression faster on all three datasets, reducing average time cost from 1331 to
207 seconds on LiveMath, from 1082 to 332 seconds on Spreadsheet, and from
587 to 318 seconds on BFCL-V4. Averaged across datasets, \method requires
286 seconds,  corresponding to
a \textbf{3.5$\times$ speedup}.

SkillReducer uses fewer direct compression-model calls, but additionally
requires 40--80 task rollouts for candidate validation and repair. In
contrast, \method uses several structured LLM calls but
\textbf{requires no task rollout on any dataset}. The results indicate that
environment interaction, rather than the number of compression calls alone,
dominates the end-to-end cost of evaluation-guided compression. Consequently,
the evaluation-free design of \method substantially reduces latency while
avoiding potential overfitting and dependence on executable tasks and behavioral verifiers.

\begin{takeawaybox}
\method achieves a \textbf{3.5$\times$ average speedup} over SkillReducer while requiring \textbf{zero task rollouts}, highlighting the efficiency.
\end{takeawaybox}

\subsection{RQ4: Cross-Model Generalization}
\label{sec:cross_model_generalization}

\begin{figure}[htbp]
    \centering
    \includegraphics[width=0.95\columnwidth]{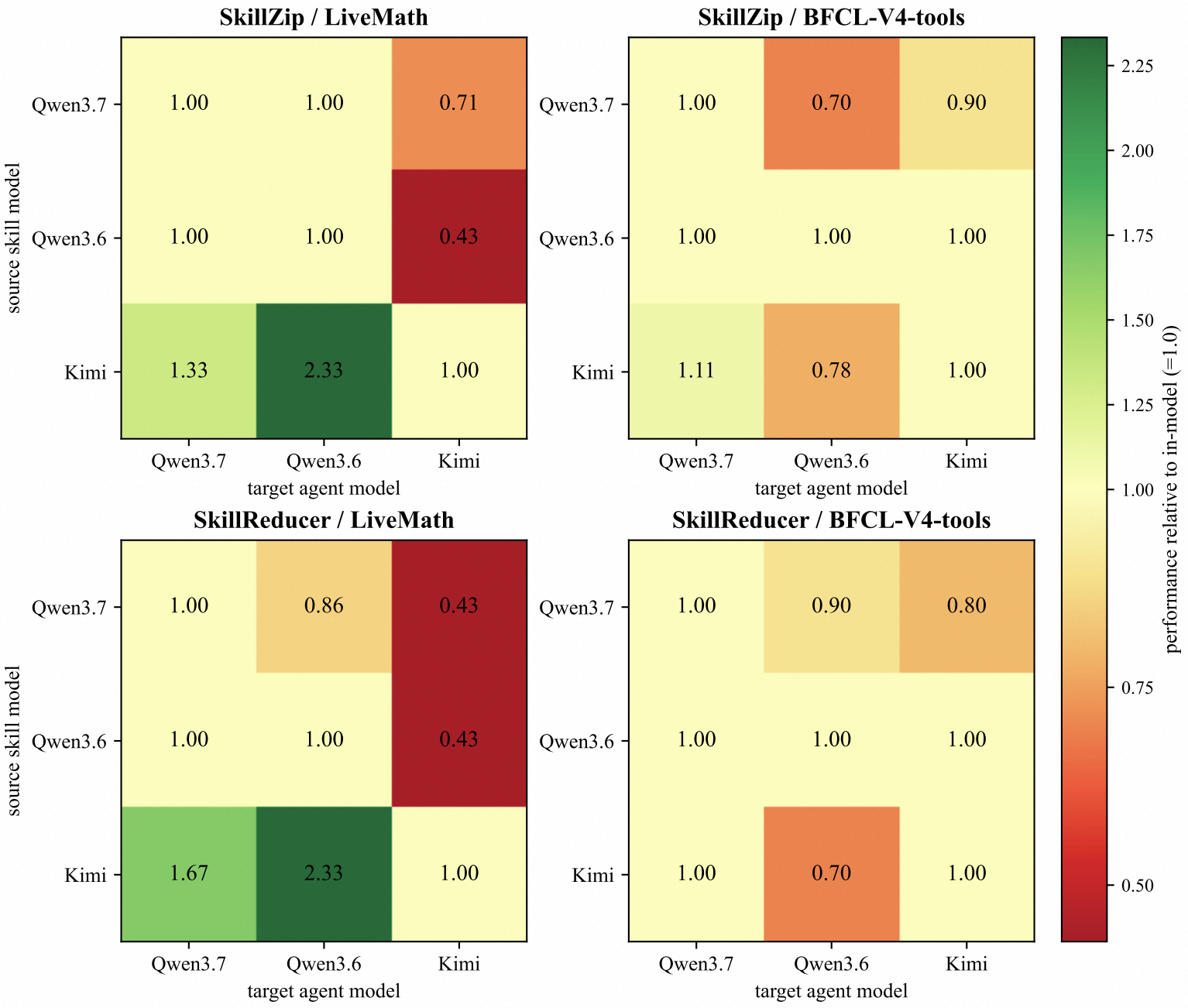}
    \caption{Cross-model generalization of compressed skills on LiveMath and
    BFCL-V4. Rows denote the source model whose evolved skill is
    compressed, and columns denote the target model executing the compressed
    skill. Diagonal cells represent same-model deployment, while off-diagonal
    cells measure cross-model transfer.}
    \label{fig:cross_model_transfer}
\end{figure}

Fig.~\ref{fig:cross_model_transfer} evaluates whether a skill compressed from
one source model can be executed by a different target model. On LiveMath,
\method achieves an overall retention of \textbf{0.97}, compared with 0.91 for
SkillReducer. Since their same-model results are comparable, the improvement
mainly comes from off-diagonal source--target pairs, suggesting that preserving
explicit rules, guards, and output constraints produces a more
model-independent skill representation.


\begin{takeawaybox}
\method transfers compressed skills across agent backbones
without target-specific evaluation, improving retention on LiveMath
 and remaining comparable to SkillReducer on
BFCL-V4.
\end{takeawaybox}

\subsection{RQ5: Continual Zip-on-Write Compression}
\label{subsec:zipwrite}

\begin{figure*}[htbp]
  \centering
  \includegraphics[width=0.95\textwidth]{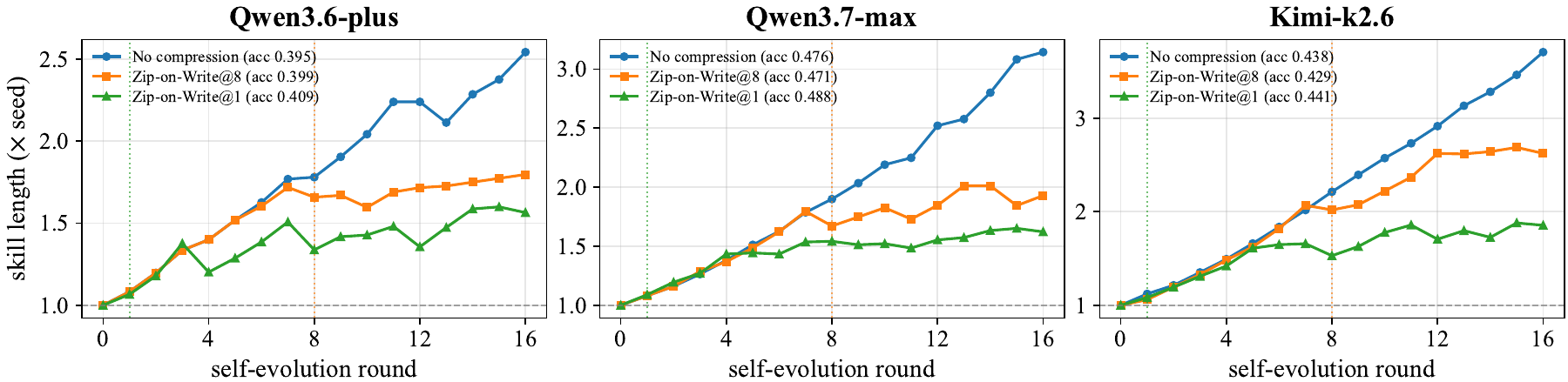}
  \caption{Skill length (by $N\times$) during self-evolution on LiveMath for
  three agent backbones. Each panel compares no compression against
  \textsc{Zip-on-Write} continual compression activated at round~8 and at round~1; legends report the final test accuracy.}
  \label{fig:zipwrite}
\end{figure*}

Fig.~\ref{fig:zipwrite} evaluates \textsc{Zip-on-Write}, the continual mode of
\textsc{SkillZip}, inside a 16-round self-evolution loop on LiveMath with three backbones. Without compression, the SkillOpt evolver exhibits the bloat pathology consistently across all three models: the length of skill grows monotonically to $2.5\times$, $3.1\times$, and $3.7\times$ its seed length, respectively. Activating \textsc{Zip-on-Write} from round~1 \emph{bounds} this growth for the entire trajectory, capping the skill at roughly $1.6\times$--$1.9\times$ across all three models---a $38\%$--$50\%$ reduction relative to the uncompressed endpoint---because each write is immediately absorbed into the typed contract library and periodically repacked.

Two further observations yield the practical guidance. First, activation time
matters: switching compression on only at round~8 recovers part of the
accumulated redundancy but never catches up with the early-activation
trajectory (e.g., $2.6\times$ vs.\ $1.9\times$ on Kimi-k2.6), showing that
redundancy is cheaper to prevent than to remove. Second, compression does not
trade accuracy for compactness: the final held-out test accuracy of the
round-1 configuration matches or slightly exceeds the uncompressed skill on all three backbones. 

\begin{takeawaybox}
Continual compression should be on from the start of evolution---it \emph{keeps the skill within an obviously smaller length factor of the uncompressed one at nearly no accuracy cost}, whereas \emph{a later activation of compression only partially undoes bloat that has already compounded}.
\end{takeawaybox}


\section{Conclusion}
\label{sec:conclusion}

Self-evolving agents need a mechanism for forgetting repetition without forgetting procedure. \method treats a skill as a typed contract and compresses it using a simple principle: explain shared structure once, reference it where needed, and keep genuine differences explicit. The resulting shortest-faithful-explanation objective is evaluation-free, protects rare requirements through hard coverage, and unifies rule sharing, scope placement, workflow reuse, and exception handling. One-shot compression requires one structured extraction call followed by deterministic optimization; Zip-on-Write integrates the same principle into continual skill evolution. The proposed experiments separate structural correctness, held-out behavior, evaluation-set overfitting, and cost, providing a practically reliable skill compression method tailored for self-evolving agents.

\FloatBarrier
\IEEEtriggeratref{32}
\bibliographystyle{IEEEtran}
\bibliography{references}

\clearpage
\onecolumn
\appendices
\section{Additional Theoretical Details}
\label{app:theory-details}

\begin{roadmapbox}
The public appendix records the decision-specific cost tests and proofs behind the shortest-faithful-explanation objective, the implementation and continual-update protocol, the full contract schema and prompts, the experimental protocol, and additional analysis of compression behavior. It is organized as a reproducibility companion to the main paper rather than as a space-constrained supplement.
\end{roadmapbox}

\subsection{Practical Length Model}

The abstract length in Eq.~\eqref{eq:objective} is instantiated as
\begin{equation}
L(x)=|\operatorname{Render}(x)|_{\mathrm{tok}}
+\gamma_{\mathrm{def}}(x)
+\gamma_{\mathrm{ref}}(x)
+\gamma_{\mathrm{scope}}(x),
\label{eq:practical_code}
\end{equation}
where the first term is the rendered token length and the remaining terms
charge for defining a named abstraction, referencing it, and expressing its
scope. These charges prevent degenerate solutions that introduce many tiny
abstractions whose notation costs more than the text they replace.

\subsection{Decision-Specific Cost Tests}

\paragraph{Equivalent requirements.}
Suppose clauses $x_1$ and $x_2$ can be represented by a common typed unit $z$.
Sharing is selected when
\begin{equation}
L(z)+L(x_1\mid z)+L(x_2\mid z)
<
L(x_1)+L(x_2).
\label{eq:merge_test}
\end{equation}
For true paraphrases, the residual terms are nearly empty. A polarity change,
guard difference, tool-argument difference, or output-field difference must be
encoded in the residual and can make sharing unprofitable.

\paragraph{Scope lifting.}
Let rule $c$ occur in child scopes $s_1,\ldots,s_r$. It may be moved to their
closest common ancestor $u$ only if every relevant path from $u$ requires the
rule and all local conflicts are encoded as exceptions. Among feasible
placements, lifting is selected when
\begin{equation}
L(c@u)+rL(\textsf{scope-ref})
<
\sum_{i=1}^{r}L(c@s_i).
\label{eq:scope_lift}
\end{equation}

\paragraph{Workflow reuse.}
Let workflow fragment $q$ occur $r$ non-overlapping times. A shared procedure
is useful when
\begin{equation}
L(\operatorname{def}(q))
+rL(\operatorname{call}(q))
<
rL(q).
\label{eq:motif_threshold}
\end{equation}
The threshold adapts to fragment length, the number of occurrences, and the
definition/call overhead.

\paragraph{Common rule with exceptions.}
Suppose guarded clauses $c_i@g_i$ share a common core $c$. SkillZip compares
\begin{equation}
L(c)+\sum_i L(\delta_i)
\qquad\text{against}\qquad
\sum_i L(c_i@g_i),
\label{eq:exception_code}
\end{equation}
where $\delta_i$ records the guarded difference. The common representation is
selected only when it is shorter and every exception remains attached to the
rule it modifies.

\subsection{Proofs and Additional Guarantees}

\begin{proof}[Proof of Proposition~\ref{prop:coverage}]
Feasibility requires every $a\in\atoms_{\mathrm{req}}(S)$ to be covered.
Coverage may be direct, realized by a scope-safe shared rule, realized by a
shared workflow whose expansion contains the original nodes and edges, or
preserved verbatim in the residual. Removing the only representation of any
required unit violates the constraint in Eq.~\eqref{eq:objective}; therefore,
no feasible solution can discard it.
\end{proof}

\begin{proposition}[Reuse threshold]
\label{prop:reuse}
For a repeated structure $q$, introducing a shared definition is beneficial if
and only if the saved repetition cost exceeds the definition and reference
overhead.
\end{proposition}

\begin{algorithm}[htbp]
\caption{Zip-on-Write Continual Compression}
\label{alg:ziponwrite}
\begin{algorithmic}[1]
\Require State $Z_{t-1}$, accepted patch $\Delta_t$, repack policy $\eta$
\Ensure Updated state $Z_t$ and rendered skill $\widetilde S_t$
\State $B_t\gets\Call{ScanPatch}{\Delta_t}$
\State $(\atoms_t,R_t)\gets\Call{ExtractPatch}{B_t}$
\State $N_t\gets\Call{RetrieveCompatible}{Z_{t-1},\atoms_t}$
\State $H_t\gets\Call{ProposeOps}{\atoms_t,N_t}$
\Comment{absorb/refine/extend/refactor}
\State $Z'\gets\Call{LocalMinCostUpdate}{Z_{t-1},\atoms_t,R_t,H_t}$
\State \Call{AssertCoverage}{$Z',\atoms_t$}
\State \Call{UpdateReuseStatistics}{$Z'$}
\If{$\Call{RepackDue}{Z',\eta}$}
  \State $U'\gets\Call{Units}{Z'}$
  \State $R'\gets\Call{Residuals}{Z'}$
  \State $H\gets\Call{ProposeReuse}{U'}$
  \State $Z'\gets\Call{MinCostCover}{U',R',H}$
\EndIf
\State $\widetilde S_t\gets\Call{Render}{Z'}$
\If{$\Call{AuditDue}{Z'}$}
  \State $(Z',\widetilde S_t)\gets\Call{StructuralAuditAndRestore}{Z',\widetilde S_t}$
\EndIf
\State $Z_t\gets\Call{AtomicCommit}{Z',\widetilde S_t}$
\State \Return $(Z_t,\widetilde S_t)$
\end{algorithmic}
\end{algorithm}

\begin{proof}
The separate form costs $rL(q)$, whereas the shared form costs
$L(\operatorname{def}(q))+rL(\operatorname{call}(q))$. The shared form reduces
the objective exactly when Eq.~\eqref{eq:motif_threshold} holds.
\end{proof}

\begin{proposition}[Local update equivalence]
\label{prop:local_exact}
Let $K_{t-1}$ be the compact contract before patch $\Delta_t$. If the patch
introduces no profitable abstraction whose occurrences span both the affected
and unaffected regions, minimizing Eq.~\eqref{eq:objective} over the affected
type--scope neighborhood yields the same update as rerunning the batch
optimizer on the full contract.
\end{proposition}

\begin{proof}
Under the stated condition, every candidate whose feasibility or cost changes
is contained in the affected neighborhood. All candidates and cost terms
outside that neighborhood are identical before and after the patch and
therefore contribute the same constant to the local and global objectives.
Both optimizations consequently select the same change.
\end{proof}

The condition of Proposition~\ref{prop:local_exact} can fail when several
patches jointly create a repeated workflow or a rule shared across previously
unrelated scopes. This motivates periodic global repacking over the compact
contract. Repacking is not required for correctness---coverage remains
enforced after every patch---but it can recover savings invisible to purely
local updates.

\subsection{Boundary of the Guarantee}

The preservation guarantee is relative to the contract produced by the parser;
it does not establish that arbitrary natural language has been interpreted
perfectly or that two rendered skills induce identical behavior for every
language model. SkillZip makes this boundary explicit through source
provenance, parser confidence, locked residuals, and the optional independent
structural audit. Its intended conservative failure mode is
under-compression, not silent deletion of uncertain requirements.

\section{Implementation Details}
\label{app:implementation}

This appendix specifies an implementation that can be translated directly into code. The intended repository separates parsing, optimization, rendering, and evaluation so the evaluation-free claim can be audited from file access and logs.

\subsection{Repository and Command-Line Interface}
\begin{lstlisting}[style=compactjson]
skillzip/
  scanner.py          # Markdown blocks + stable IDs
  extract.py          # schema-constrained contract parse
  relations.py        # typed retrieval + NLI checks
  workflow.py         # graph and repeated-sequence mining
  optimize.py         # min-cost covering selection
  render.py           # deterministic text templates
  audit.py            # compressed-text contract diff
  online.py           # Zip-on-Write + transaction log
  schemas/skillzip.schema.json
  prompts/{extract,patch,audit}.txt
  configs/{default,efficient}.yaml
\end{lstlisting}

\begin{lstlisting}[style=compactjson]
# one-shot
skillzip compress SKILL.md \
  --config configs/default.yaml \
  --state skillzip.json \
  --output SKILL.compact.md

# one continual update
skillzip update skillzip.json PATCH.md \
  --output SKILL.md

# evaluation-free structural check
skillzip audit SKILL.compact.md skillzip.json

# explain each accepted/rejected abstraction
skillzip inspect skillzip.json --show-savings
\end{lstlisting}
The CLI writes through temporary files and replaces the persistent state only after schema validation, coverage validation, and rendering succeed.

\subsection{Deterministic Scanning and Stable Provenance}
The scanner emits blocks with fields \texttt{id}, \texttt{kind}, \texttt{heading\_path}, \texttt{text}, and \texttt{line\_range}. The identifier is a hash of normalized text and ancestor headings; line numbers are stored separately so unrelated insertions do not invalidate provenance. Code fences and tables remain atomic blocks because splitting them can destroy a template or schema.

A scope path is an array such as
\begin{lstlisting}[style=compactjson]
["root", "workflow", "if-validation-fails"]
\end{lstlisting}
Markdown nesting supplies initial scopes. Explicit guards create child scopes even when no heading is present. Promotion to a wider scope requires cited language such as ``always'', ``for every request'', or structural repetition across all relevant children.

\subsection{Relation Checking Pipeline}
Hard compatibility keys are applied before embeddings:
\begin{equation}
(\textsf{type},\textsf{modality},\textsf{tool/output namespace},\textsf{scope family}).
\end{equation}
Exact normalized units merge by hash. Remaining units are embedded once in a separate index per key. The top-$k$ candidates are passed to a frozen cross-encoder that returns equivalence, left implication, right implication, conflict, or unrelated. Only equivalence and implication create sharing candidates. Conflict creates an exception edge; low-confidence pairs remain separate.

Cache keys include normalized texts, types, scope signatures, model and revision, thresholds, and prompt hashes. This supports deterministic reruns and avoids repeated relation calls in Zip-on-Write.

\subsection{Practical Length Model}
For a rendered unit or abstraction $x$,
\begin{equation}
\widehat L(x)=|\operatorname{Render}(x)|_{\mathrm{tok}}
+\lambda_d n_{\mathrm{def}}(x)
+\lambda_r n_{\mathrm{ref}}(x)
+\lambda_s d(x),
\end{equation}
where $d(x)$ is scope depth. The default $\lambda$ values equal the actual token cost of the corresponding template delimiters. They are not fitted on downstream tasks. Every candidate log stores separate before, definition, reference, exception, residual, and after costs.

\begin{lstlisting}[style=compactjson]
{
  "candidate": "workflow:validate-repair",
  "before_tokens": 94,
  "definition_tokens": 32,
  "reference_tokens": 17,
  "exception_tokens": 3,
  "after_tokens": 52,
  "saving_tokens": 42,
  "covered_units": ["W12", "W13", "W31", "W32"]
}
\end{lstlisting}

\subsection{Scope Optimization and Workflow Packing}
For each normalized rule, a bottom-up dynamic program compares keeping copies in child scopes with placing one copy at an ancestor. A placement is infeasible when a reachable child does not require the rule or contains an unencoded conflict. The DP returns the minimum-cost feasible placement.

Workflow paths are represented by action identifiers that include action type, tool name, required arguments, and guard class. Re-Pair proposes repeated adjacent pairs; prefix/suffix tries identify shared branch segments. A candidate must have at least two non-overlapping occurrences, compatible entry/exit behavior, and positive saving under Eq.~\eqref{eq:motif_threshold}. Candidate selection is a weighted set-packing problem. The default implementation greedily selects saving per covered token, then applies pairwise exchange. An exact integer program is provided for small synthetic instances to measure the approximation gap.

\subsection{Structural Audit and Conservative Recovery}
The audit parser does not see the original skill or expected contract. It independently parses the rendered skill, after which a deterministic diff checks trigger polarity, guards, modality, workflow reachability, tool arguments, and output fields. For each missing element, recovery restores the shortest original source span covering that element and marks it \texttt{locked=true}. Future updates may add related content but cannot delete a locked span without explicit user approval.

\subsection{Atomic Online Updates}
Zip-on-Write uses a write-ahead log:
\begin{enumerate}[leftmargin=*,topsep=2pt,itemsep=1pt]
  \item parse and validate the patch;
  \item record proposed operations and savings;
  \item apply them to a copy of the sidecar;
  \item validate coverage and JSON schema;
  \item render a temporary skill;
  \item atomically replace state and text; and
  \item commit the log entry.
\end{enumerate}
A crash before the final step leaves the previous skill unchanged. Repacking is performed in a separate transaction so failure cannot corrupt patch ingestion.


\section{Contract Schema and Prompts}
\label{app:schema}

\subsection{Core Sidecar Schema}
\begin{lstlisting}[style=compactjson]
{
  "interface": {
    "name": "string",
    "purpose": "string",
    "triggers": [{"text":"string","spans":["B1"]}],
    "exclusions": [{"text":"string","spans":["B2"]}]
  },
  "scopes": [{
    "id": "scope-id",
    "parent": "scope-id|null",
    "guard": "string|true",
    "source_blocks": ["B3"]
  }],
  "workflow": [{
    "id": "W1",
    "scope": "scope-id",
    "kind": "action|decision|loop|fallback|stop",
    "action": "string",
    "tool": "string|null",
    "required_args": ["string"],
    "next": ["W2"],
    "spans": ["B4"],
    "confidence": 0.0
  }],
  "rules": [{
    "id": "C1",
    "scope": "scope-id",
    "modality": "must|must_not|prefer",
    "predicate": "string",
    "guard": "string|true",
    "spans": ["B5"],
    "locked": false,
    "confidence": 0.0
  }],
  "output": {
    "type": "string",
    "fields": [{
      "name":"string",
      "required":true,
      "validation":"string",
      "spans":["B6"]
    }],
    "termination": ["string"]
  },
  "evidence": [{
    "kind": "example|rationale|template|background",
    "supports": ["W1", "C1"],
    "adds_unique_unit": false,
    "source_blocks": ["B7"]
  }],
  "shared_procedures": [{
    "id": "P1",
    "actions": ["W1", "W2"],
    "occurrences": [["W1", "W2"], ["W9", "W10"]],
    "saving_tokens": 24
  }],
  "residual": [{
    "source_block": "B8",
    "reason": "AMBIGUOUS|UNIQUE_EVIDENCE|USER_LOCKED"
  }]
}
\end{lstlisting}

\subsection{One-Shot Extraction Prompt}

\begin{promptbox}[System Prompt]
You are a parser for reusable agent skills. Convert the numbered source blocks into the supplied contract schema. Extract only requirements supported by cited block IDs. Preserve negation, quantifiers, branch guards, action order, tool names, required arguments, error handling, and output fields. Do not compress or summarize. An example adds a unique unit only when it is the sole source of a branch, command, constraint, field, value restriction, or formatting rule. If type or scope is uncertain, place the source block in residual with reason AMBIGUOUS. Return JSON only.
\end{promptbox}




The host rejects unknown source IDs, uncited units, dangling workflow edges, output fields without evidence, and required arguments not present in the cited block.

\subsection{Patch Prompt}

\begin{promptbox}[System Prompt]
Parse the proposed skill patch into the same contract schema. For each unit, identify the narrowest supported scope and cite patch block IDs. Compare only with the supplied existing units. Propose one operation: ABSORB, REFINE, EXTEND, or REFACTOR. Preserve every conflict as an explicit guarded exception. Do not edit unrelated scopes. Return JSON only; the host recomputes all savings and decides whether to apply the operation.
\end{promptbox}



\subsection{Independent Audit Prompt}
\begin{promptbox}[System Prompt]
Parse this rendered skill into the contract schema. Do not use the original skill or an expected answer. Preserve exact triggers, exclusions, guards, modalities, tool arguments, workflow edges, required output fields, and termination conditions. Return JSON only.
\end{promptbox}



\subsection{Prior and Post-Compressed Skill Example}
The example shows a LiveMath skill after SkillOpt-based self-evolution
for qwen3.6-plus; one-shot \textsc{SkillZip} compresses it from
\textbf{936 to 638 tokens (32\% saving)} \textbf{with nearly perfect preservation of performance}. Over self-evolution rounds, the SkillOpt evolver repeatedly re-appended
some overlapping guidance, making the skill carry \textit{redundancy}: the
``output only the single valid value / discard extraneous roots'' rule and the
``verify each intermediate product and sum'' rule each appear \textbf{twice},
and two near-identical ``verify each candidate against all constraints'' rules
coexist. \textsc{SkillZip} \textit{removes these duplicates, folds the two
near-paraphrases into one, and normalizes the free-text \texttt{Approach} into a
compact numbered \texttt{Workflow}}, while \textbf{preserving every distinct
output contract verbatim in meaning}---the \texttt{\textbackslash boxed\{\}}
format, set/expression notation, ascending-sort, factorial-magnitude, and
extraneous-root-discarding requirements all survive. Because compression
\textit{only re-expresses or de-duplicates content already covered by the
extracted contract}, the result is \textbf{strictly shorter yet performance
equivalent}.

The following is the uncompressed skill.
\begin{lstlisting}[style=compactjson]
## Name
Competition-style math solver.

## Description
Solve competition-style mathematics problems accurately and return the final
answer in the exact required format.

## When to use
- Any single-answer math problem: algebra, number theory, combinatorics,
  geometry, calculus, probability, or discrete math.

## Approach
1. Restate the problem in your own words; list every given quantity and the
   exact quantity to find.
2. Classify the problem and choose a suitable method (direct computation, algebraic manipulation, casework, invariants, symmetry, recursion, or a known theorem) and briefly justify the choice.
3. Introduce explicit notation for all unknowns and constraints before computing.
4. Solve step by step, keeping exact values (fractions, radicals, symbolic
   constants such as pi or e) rather than rounding early.
5. Verify the result: substitute back into the original relations, sanity-check
   magnitude and units, and confirm every constraint and edge case holds.

## Rules
- Never round intermediate results unless the problem explicitly asks for a
  decimal approximation.
- Reduce fractions to lowest terms and rationalize denominators where standard.
- For counting problems, decide ordered vs unordered and with vs without
  replacement before summing, and enumerate disjoint cases exhaustively.
- State any assumption you make when the problem is ambiguous, then proceed with the most standard interpretation.

## Output
- Show concise working, then put the final answer on its own line as \boxed{...}.
- Give the answer in simplest exact form (e.g., 1/2 not 0.5) unless a decimal is explicitly requested.
- For a set answer use set notation \{...}; for an expression, output the
  simplified expression itself rather than a numeric approximation.
- If the problem requests multiple answers, list all distinct values separated by commas inside a single \boxed{} command (e.g., \boxed{a, b}).
- If the problem asks for a single specific value (e.g., maximum, minimum, unique solution) or implies a unique answer, output ONLY that single valid value. Strictly discard any extraneous candidates, negative roots, or intermediate results that do not satisfy all constraints (such as positivity or domain restrictions); do not list them in the final \boxed{}.
- For problems requesting multiple answers, strictly verify that each listed value satisfies all original problem constraints and definitions; discard any extraneous roots or candidates that fail validation. When listing valid multiple answers, strictly sort them in ascending numerical order before placing them in \boxed{}.
- For problems involving factorials, permutations, or large integer products, explicitly compute the magnitude or leave in factorial notation; never default to 1 or small integers unless rigorously derived.
- For multi-step arithmetic or large number calculations, explicitly verify each intermediate product and sum to prevent off-by-one or digit errors.
- If the problem asks for a single specific value (e.g., maximum, minimum, unique solution) or implies a unique answer, output ONLY that single valid value. Strictly discard any extraneous candidates, negative roots, or intermediate results that do not satisfy all constraints (such as positivity or domain restrictions); do not list them in the final \boxed{}.
- For problems with multiple valid answers, strictly verify each candidate against all problem constraints (e.g., positivity, integer requirements, domain definitions) and discard any that fail validation. If a specific subset of solutions is requested (e.g., positive integers), exclude all others.
- For multi-step arithmetic or large number calculations, explicitly verify each intermediate product and sum to prevent off-by-one or digit errors.
\end{lstlisting}

The following is the corresponding compressed skill by \method.
\begin{lstlisting}[style=compactjson]
## Name
Competition-style math solver.

## Description
Solve competition math accurately; return exact required format.

## When to use
- Any single-answer math problem: algebra, number theory, combinatorics, geometry, calculus, probability, discrete math.

## Rules
- Never round intermediate results unless problem explicitly asks decimal approximation.
- Reduce fractions lowest terms; rationalize denominators where standard.
- For counting problems, decide ordered vs unordered and with vs without replacement before summing; enumerate disjoint cases exhaustively.
- State any assumption you make when the problem is ambiguous, then proceed with the most standard interpretation.

## Workflow
1. Restate problem; list every given and exact quantity to find.
2. Classify problem; choose suitable method (direct computation, algebraic manipulation, casework, invariants, symmetry, recursion, known theorem); justify.
3. Introduce explicit notation all unknowns and constraints before computing.
4. Solve step by step; keep exact values (fractions, radicals, symbolic constants such as pi or e) not rounding early.
5. Verify result: substitute original relations, sanity-check magnitude and units, confirm every constraint and edge case holds.

## Output
- Show concise working; put final answer own line as \boxed{...}.
- Give simplest exact form (e.g., 1/2 not 0.5) unless decimal explicitly requested.
- Set answer use \{...\}; expression output simplified expression itself not numeric approximation.
- If problem requests multiple answers, list all distinct values separated commas inside single \boxed{} command (e.g., \boxed{a, b}).
- When problem asks single specific value or implies unique answer, output ONLY that single valid value; strictly discard extraneous candidates, negative roots, intermediate results not satisfying all constraints (e.g., positivity, domain restrictions); do not list final \boxed{}.
- When problems involve factorials, permutations, large integer products, explicitly compute magnitude or leave factorial notation; never default 1 or small integers unless rigorously derived.
- When multi-step arithmetic or large number calculations, explicitly verify each intermediate product and sum.
- Strictly verify each listed value satisfies all original constraints and definitions (e.g., positivity, integer requirements, domain definitions); discard extraneous roots or failing candidates; if specific subset requested (e.g., positive integers), exclude others. When listing valid multiple answers, strictly sort ascending numerically before placing in \boxed{}.
\end{lstlisting}

\section{Additional Theory and Analysis}
\label{app:analysis}

\subsection{Why the Objective Is More Than Semantic Deduplication}
Semantic deduplication answers whether two clauses have similar meanings. The shortest-explanation objective asks a different question: after paying for a shared definition, references, scope notation, and exceptions, is the shared representation actually shorter? It can reject a plausible merge when exception encoding is verbose, and it can accept a workflow abstraction whose occurrences are far apart in the document but share the same guarded action pattern. It also compares alternative abstractions that compete for the same source units.

\subsection{Optimization Decomposition}
For a fixed candidate set, the optimizer separates into several subproblems. Equivalent-unit clustering uses union--find after conflict filtering. Rule placement is a tree dynamic program. Workflow selection is weighted set packing because overlapping procedures cannot both replace the same action occurrence. The implementation uses greedy selection with pairwise exchange and provides an exact integer program for benchmark instances. The exact solver measures the approximation gap but is not used in the main efficient configuration.


\subsection{Boundary of the Guarantee}
Proposition~\ref{prop:coverage} protects the contract produced by the parser; it does not prove that a language model has perfectly interpreted arbitrary natural language. \method makes this limitation visible through cited source blocks, parser confidence, locked residuals, independent structural audit, and per-type parser evaluation. The intended failure mode is to retain too much text, not to silently delete an uncertain rule.

\subsection{Idempotence Under a Stable Parse}
For fixed parser output, candidate set, length model, renderer, and deterministic tie breaking, let $\mathsf{Zip}(S)$ be the selected rendered skill. If rendering and reparsing recover the same covered contract, a second pass introduces no new candidate that was absent from the first minimization. Therefore $\mathsf{Zip}(\mathsf{Zip}(S))=\mathsf{Zip}(S)$. This property will be tested empirically by applying one-shot compression twice and reporting token and contract differences.

\end{document}